\documentclass[conference]{IEEEtran}
\IEEEoverridecommandlockouts
\usepackage{cite}
\usepackage{amsmath,amssymb,amsfonts}
\usepackage{algorithmic}
\usepackage{graphicx}
\usepackage{textcomp}
\usepackage{xcolor}
\usepackage{url}
\usepackage{booktabs}

\usepackage{afterpage}
\usepackage{amsmath, amsthm} 
\newtheorem{definition}{Definition}

\newtheorem{theorem}{Theorem}

\def\BibTeX{{\rm B\kern-.05em{\sc i\kern-.025em b}\kern-.08em
    T\kern-.1667em\lower.7ex\hbox{E}\kern-.125emX}}

\begin{document}

\title{ProTerrain: Probabilistic Physics-Informed Rough Terrain World Modeling \\ 

\thanks{This work was supported in part by the Finnish Doctoral Program Network in Artificial Intelligence (AI-DOC), under the Grant
VN/3137/2024-OKM-6; and in part by Horizon Europe Project XSCAVE under Grant 101189836.}
\author{
\IEEEauthorblockN{
Golnaz Raja\IEEEauthorrefmark{1},
Ruslan Agishev\IEEEauthorrefmark{2},
Milo\v{s} Pr\'agr\IEEEauthorrefmark{1},
Joni Pajarinen\IEEEauthorrefmark{4},
Karel Zimmermann\IEEEauthorrefmark{2},\\
Arun Kumar Singh\IEEEauthorrefmark{3},
Reza Ghabcheloo\IEEEauthorrefmark{1}
}
\IEEEauthorblockA{%
\IEEEauthorrefmark{1}Tampere University\quad
\IEEEauthorrefmark{2}Czech Technical University\quad
\IEEEauthorrefmark{3}University of Tartu\quad
\IEEEauthorrefmark{4}Aalto University}}}

\maketitle

\begin{abstract}

Uncertainty-aware robot motion prediction is crucial for downstream traversability estimation and safe autonomous navigation in unstructured, off-road environments, where terrain is heterogeneous and perceptual uncertainty is high. Most existing methods assume deterministic or spatially independent terrain uncertainties, ignoring the inherent local correlations of 3D spatial data and often producing unreliable predictions. In this work, we introduce an efficient probabilistic framework that explicitly models spatially correlated aleatoric uncertainty over terrain parameters as a probabilistic world model and propagates this uncertainty through a differentiable physics engine for probabilistic trajectory forecasting. By leveraging structured convolutional operators, our approach provides high-resolution multivariate predictions at manageable computational cost. Experimental evaluation on a publicly available dataset shows significantly improved uncertainty estimation and trajectory prediction accuracy over aleatoric uncertainty estimation baselines. 

\end{abstract}


\section{Introduction}

Robust robot motion prediction in unstructured, off-road environments is essential for reliable downstream assessment of traversability, where terrain heterogeneity, complex robot–terrain interactions, and substantial perceptual uncertainty present significant challenges.~\cite{shu2024overview,papadakis2013terrain}. Traditional methods focus on traversabilitiy estimation by building deterministic geometric maps (e.g., height, slope, roughness) from exteroceptive sensors~\cite{shu2024overview,guastella2020learning}, but typically ignore key physical terrain parameters, spatial correlation, and environmental uncertainty required for reliable off-road navigation~\cite{vasudevan2009gaussian,russell2021multivariate}.

\begin{figure}[t]
  \centering
  \includegraphics[width=\columnwidth]{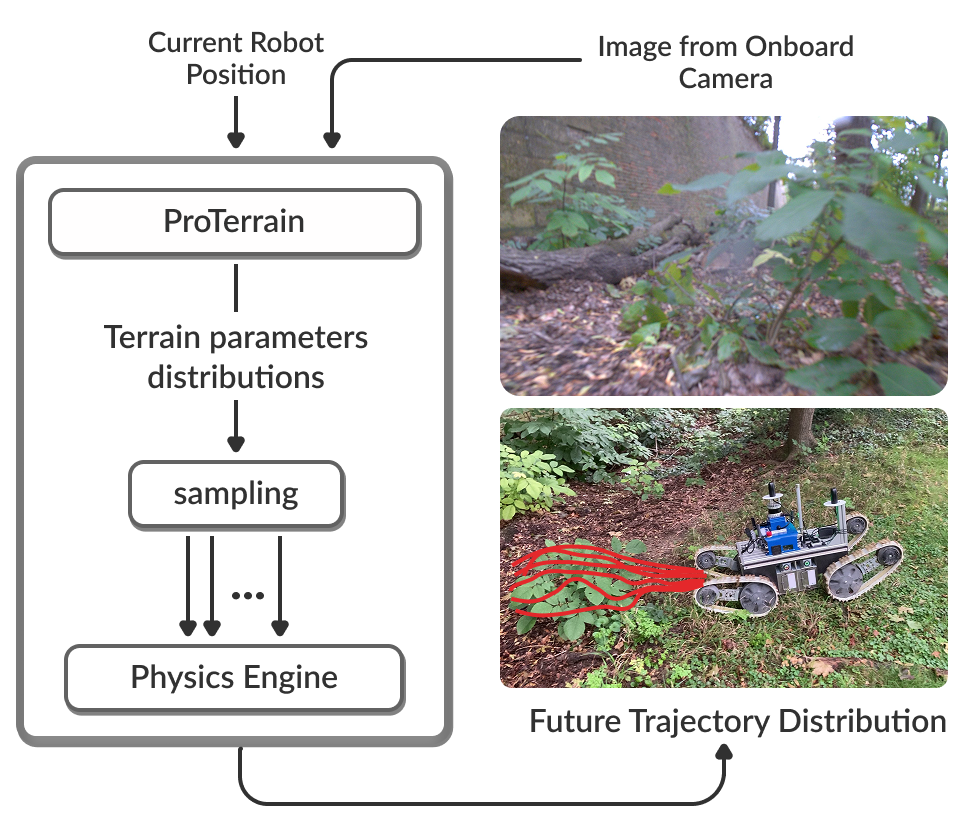}
  \caption{\textbf{Architecture overview}: Uncertainty‐aware terrain parameter estimation chained with differentiable physics.}
  \label{fig:general_pipeline}
\end{figure}

\afterpage{
\begin{figure*}[t]
  \centering
  \includegraphics[width=\textwidth,keepaspectratio]{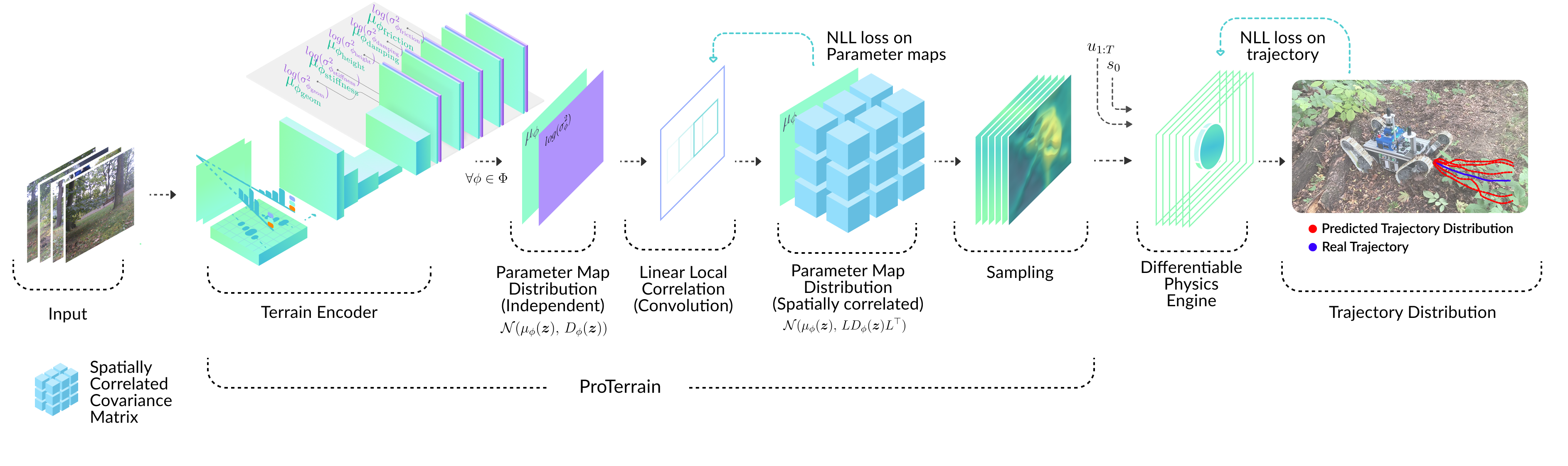}
  \caption{%
\textbf{Overview of the proposed probabilistic world model architecture.}
Given onboard camera images as input, the Terrain Encoder employs a Lift-Splat-Shoot (LSS) module to project features into bird’s-eye-view grids by estimating per-ray depth distributions and vertically integrating features. These BEV features are then processed by convolutional layers to predict mean and variance maps for multiple terrain parameters (e.g., geometric height, stiffness, friction). To capture spatially correlated uncertainty, the predicted variance maps are convolved with a fixed Gaussian kernel, forming structured multivariate Gaussian distributions over terrain parameters. Samples from the formed probabilistic world model are subsequently used as input to a differentiable physics engine for downstream probabilistic trajectory prediction.
}
  \label{fig:detailed pipeline}
\end{figure*}
}

Physics-informed terrain modeling has improved robot motion estimation by coupling perception with differentiable physics engines~\cite{agishev2024monoforce}, but most existing models treat map uncertainty as negligible—an assumption that limits safety and reliability. Recent approaches incorporate per-cell uncertainty~\cite{shaban2022semantic}, but often assume spatial independence, neglecting the inherent correlations from sensor noise, occlusion, and true surface variability, which leads to unrealistic samples and poor risk estimates~\cite{vasudevan2009gaussian}. Capturing full spatial correlation remains an open challenge due to the prohibitive computational and memory demands of high-dimensional covariance estimation. As a result, despite the recognized need for both physics-informed modeling and richer representations of uncertainty in off-road environments~\cite{shu2024overview, cai2025pietra}, explicit and tractable modeling of spatially correlated uncertainty within end-to-end, differentiable motion forecasting pipelines remains largely unaddressed.

We propose \emph{ProTerrain}, an end-to-end probabilistic framework that explicitly models spatially correlated aleatoric uncertainty in terrain parameters and propagates this uncertainty via a differentiable physics engine for probabilistic trajectory forecasting (see Fig.~\ref{fig:general_pipeline} for an overview). Leveraging structured convolutional operators for tractable, high-resolution prediction, ProTerrain achieves improved calibration and predictive accuracy over prior uncertainty baselines. \textbf{Key contributions} include: (i) end-to-end probabilistic world modeling of terrain with spatially correlated aleatoric uncertainty; (ii) scalable loss formulation based on structured covariance estimation via convolution, enabling tractable and high-resolution multivariate probabilistic predictions; (iii) uncertainty-aware trajectory prediction via differentiable physics; and (iv) extensive real-world validation demonstrating superior accuracy and calibration compared to uncertainty estimation baselines.

The paper is organized as follows: Section~\ref{Sc. Related work} surveys related work and outlines our positioning. Section~\ref{sec:problem} formulates the problem. Section~\ref{Sc.background} reviews relevant background and preliminaries. Section~\ref{sec:method} presents our methodology for enforcing spatial correlation and constructing the probabilistic world model. Section~\ref{Sc. eval} details the qualitative and quantitative evaluation. Finally, Section~\ref{Sc.conclusion} concludes with discussion and outlook.

\section{Related Work}
\label{Sc. Related work}

In this section, we provide an overview of traversability estimation methods. Existing methods can be broadly devided into three categories: (i) geometric terrain mapping, (ii) probabilistic terrain mapping, and (iii) physics-informed and dynamics-aware terrain modeling.

\subsection{Geometric Terrain Mapping}

Classical approaches estimate traversability by constructing geometric maps—such as elevation, slope, and roughness—from onboard exteroceptive sensors like LiDAR and RGB(-D) cameras~\cite{shaban2022semantic, schilling2017geometric, guan2023tnes}. However, such models primarily capture surface shape, neglecting crucial physical terrain attributes (e.g., compliance, friction, load-bearing capacity) that directly influence traversability. Environmental variability is typically encoded only through semantic classes or assigned costs, which fails to account for the specific interaction between a robot's dynamics and the terrain. Moreover, classical geometric approaches generally lack mechanisms for estimating or propagating uncertainty—a key limitation when dealing with ambiguous or novel environments.

\subsection{Probabilistic Terrain Mapping}
A widely adopted approach for incorporating uncertainty in traversability analysis is to augment geometric maps (elevation or cost) with per-cell, independently estimated variances. For example,~\cite{Fankhauser2018ProbabilisticTerrainMapping} propagates sensor and localization uncertainty into local per-cell confidence bounds. Learning-based methods~\cite{triest2024unrealnet,yang2022real} predict traversability features and associated aleatoric uncertainty by outputting a Gaussian for each map cell, usually trained with the negative log-likelihood loss of~\cite{kendall2017uncertainties}; similar independent-factor models are used in risk map construction~\cite{fan2021step,fan2021learning}. This family of methods, however, overlooks the inherent spatial correlation found in real terrain.

To address spatial dependencies, Gaussian Processes (GPs) have been used for terrain modeling, including elevation~\cite{vasudevan2009gaussian}, geometry-slip correlation~\cite{endo2022active}, and raw LiDAR mapping~\cite{hansen2023range}. However, full GP regression is computationally demanding and does not scale to real-time, high-resolution mapping.

Alternative Bayesian approaches, such as Bayesian Generalized Kernel (BGK) inference~\cite{shan2018bayesian,xue2023traversability}, offer more efficient probabilistic modeling from LiDAR data. Yet, these methods are generally limited to LiDAR-only settings, reducing robustness under adverse sensing (e.g., fog, dust) and constraining richer semantic understanding available from vision or multimodal sources.

\subsection{Physics-Informed and Dynamics-Aware Approaches}

Physics-informed and dynamics-aware methods address the core challenge of traversability by explicitly modeling robot--terrain interaction through physics-related costmaps. Approaches range from extending geometric maps to include estimates of mechanical effort~\cite{carvalho20243d}, to leveraging simulation for empirical traversability assessment, with supervision from physical outcomes such as slippage or failure events~\cite{frey2022locomotion}. Recent deep learning approaches leverage self-supervised force-torque or inertial feedback to capture terrain difficulty, and evidential learning to infer parameters with quantified uncertainty under physics-informed priors~\cite{wellhausen2019should, castro2022does, cai2025pietra}; however, these methods typically address only a subset of physical properties and do not provide end-to-end robot--terrain interaction modeling.

In contrast, frameworks like MonoForce~\cite{agishev2024monoforce} integrate differentiable physics simulators within the learning loop, enabling end-to-end training of all key terrain parameters for accurate robot motion prediction from perception. These state-of-the-art systems learn directly from trajectories in a self-supervised fashion, but remain fundamentally deterministic in their treatment of uncertainty.

\subsection{Summary and Positioning}

While prior work has explored geometric, uncertainty-aware, and physics-informed traversability estimation, few methods integrate all three aspects in a unified framework. In particular, efficient modeling of spatially correlated uncertainty—essential for realistic 3D terrain reasoning—remains an open challenge. Our work addresses this gap by introducing a probabilistic, self-supervised world model that jointly captures physics-informed terrain parameters and their spatially correlated aleatoric uncertainty in a fully end-to-end manner.

\section{Problem Formulation}
\label{sec:problem}

Our goal is to learn an end-to-end probabilistic \emph{world model} of rough 
terrain environments that predicts robot future trajectories while explicitly 
capturing input-dependent aleatoric uncertainty in both terrain parameters 
and trajectory forecasts.

\subsection{Probabilistic World Model Definition}
Let $\boldsymbol{z} \in \mathbb{R}^{N \times H \times W \times 3}$ denote the robot's onboard RGB perception, where $N$ is the number of cameras, and $H,W$ are the camera resolution. Our central object of interest is the spatial \emph{world model}---a set of latent per-location terrain parameters that govern traversability and robot-terrain interactions.

We define a finite set of terrain parameters $\Phi := \{\phi_{1},\phi_{2},...,\phi_{P}\}$, with each $\phi_p$ representing environmental properties such as geometric height, friction, stiffness, damping, or compactness. 

At any time $t$, the world model consists of spatial maps for each terrain parameter $\phi \in \Phi$, where the uncertainty for each parameter is modeled as an \emph{input-dependent} (i.e., conditional on $\boldsymbol{z}$) and \emph{aleatoric} spatially correlated Gaussian distribution. Note that parameters are assumed conditionally independent:
\begin{equation}
\label{eq:multivariate_gaussian}
     W_t = \left\{ p(\phi \mid \boldsymbol{z}) \right\}_{\phi \in \Phi}, \quad \text{with} \quad p(\phi \mid \boldsymbol{z}) = \mathcal{N}\left({\mu}_\phi(\boldsymbol{z}),\, \Sigma_\phi(\boldsymbol{z})\right).
\end{equation}

Here, $W_t$ is our probabilistic world model, ${\mu}_\phi(\boldsymbol{z}) \in \mathbb{R}^{H' \times W'}$ the predicted mean map, and $\Sigma_\phi(\boldsymbol{z}) \in \mathbb{R}^{(H'W') \times (H'W')}$ the predicted spatially correlated covariance for each terrain parameter $\phi$, both as functions of $\boldsymbol{z}$, capturing \emph{aleatoric} uncertainty---e.g., from sensor noise or ambiguous terrain. $H'$ and $W'$ denote the spatial dimensions.

To fit this probabilistic model, we will  minimize the multivariate negative log-likelihood (NLL) loss for each $\phi \in \Phi$ \cite{russell2021multivariate}:
\begin{equation}
\label{eq: general_multivariate_NLL}
\mathcal{L}_{\text{NLL}} = \frac{1}{2} \left[
    \log \det \Sigma(\boldsymbol{z}) +
    \boldsymbol{r}^\top \Sigma(\boldsymbol{z})^{-1} \boldsymbol{r}
\right],
\end{equation}
where $\boldsymbol{r} = \mathrm{vec}(\phi^\star - \mu(\boldsymbol{z}))$ is the residual between prediction and ground truth.

However, explicitly forming or inverting $\Sigma(\boldsymbol{z})$ is computationally infeasible for high-resolution maps, due to prohibitive memory and runtime requirements. For instance, even a moderate map of $128 \times 128$ cells would require more than $2$ GB of memory just to store the dense covariance matrix (of size $(H'W')^2$). To overcome this limitation, we employ a matrix-free formulation that enables scalable computation without ever constructing $\Sigma(\boldsymbol{z})$.

\subsection{Robot State and Controls}

We represent the robot state as 
$s = (\mathbf{x}, R, \mathbf{v}, \boldsymbol{\omega})$, 
where $\mathbf{x} \in \mathbb{R}^3$ is the position, 
$R \in SO(3)$ is the orientation, 
and $\mathbf{v}, \boldsymbol{\omega} \in \mathbb{R}^3$ are the linear and angular velocities. 
The control input is defined as 
$\mathbf{u} = (\mathbf{v}_c, \boldsymbol{\omega}_c)$, 
with $\mathbf{v}_c, \boldsymbol{\omega}_c \in \mathbb{R}^3$, 
and the control input sequence is denoted by 
$\mathbf{u}_{1:T}$, 
where each $\mathbf{u}_t$ encodes linear and angular velocity commands.

\subsection{Probabilistic Trajectory Prediction}
Given the estimated probabilistic world model $W_t$, which encodes spatially correlated distributions over terrain parameters, together with the initial robot state ${s}_0$ and a sequence of control commands $\mathbf{u}_{1:T}$, we predict possible future robot position trajectories via a differentiable physics engine. Assuming the distribution over trajectories is Gaussian, the input-dependent aleatoric uncertainty of future trajectory $\tau = \{ \mathbf{x}_k \}_{k=1}^T$ is modeled as:

\[
p\left({\tau} \mid W_t, {s}_0, \mathbf{u}_{1:T}\right) = \mathcal{N}\left( {\mu}_\tau(\cdot),\, \Sigma_\tau(\cdot) \right)
\]

\section{Preliminaries and Background}
\label{Sc.background}
In this section, we review prior methods and mathematical tools that underpin our probabilistic world model, including physics-informed robot-terrain modeling and scalable linear algebra techniques.

\subsection{MonoForce: Physics-informed World Modeling}
\label{Sc. Monoforce}

MonoForce~\cite{agishev2024monoforce} is an end-to-end, grey-box differentiable model for predicting robot--terrain interaction on both rigid and deformable terrains, using monocular RGB images, robot state, and control inputs. Its architecture consists of two main components: a deep terrain encoder and a differentiable physics engine.

\paragraph{Terrain Encoder}
The encoder predicts spatial maps of terrain properties from monocular images $\boldsymbol{z}$, including: geometric height $\phi_\mathrm{geom}$, support height $\phi_\mathrm{height}$, stiffness $\phi_\mathrm{stiffness}$, damping $\phi_\mathrm{damping}$, and friction $\phi_\mathrm{friction}$, all in $\mathbb{R}^{H' \times W'}$, representing per-cell estimates on a 2D grid over the local environment.

\paragraph{Differentiable Physics Engine}
The physics engine simulates the robot's dynamics on complex terrain using the predicted terrain parameter maps. Given the robot's initial state ${s}_0$ and a sequence of control commands $\mathbf{u}_{1:T}$, the model integrates the robot's trajectory by accounting for both rigid-body dynamics and terrain-induced contact forces.

The robot is modeled as a rigid body composed of $N$ mass points $P = \{ (\mathbf{p}_i, m_i) \}$, with dynamics governed by:

\begin{equation}
\label{eq:physics}
\begin{aligned}
    \dot{\mathbf{x}} &= \mathbf{v},              
    & \quad \dot{\mathbf{v}} &= \frac{1}{M} \sum_{i=1}^N \mathbf{f}_i, \\
    \dot{{R}} &= [\boldsymbol{\omega}] {R}, 
    & \quad \dot{\boldsymbol{\omega}} &= {J}^{-1} \sum_{i=1}^N (\mathbf{p}_i \times \mathbf{f}_i)
\end{aligned}
\end{equation}
where the total mass is $M = \sum m_i$, ${J}$ is the moment of inertia, $\mathbf{f}_i$ is the net force on each mass point, and $[\boldsymbol{\omega}] \in \mathbb{R}^{3 \times 3}$  is the skew-symmetric matrix representation of the angular velocity.

Each contact force $\mathbf{f}_i$ combines gravity and terrain reaction:
\[
\mathbf{f}_i = 
\begin{cases}
    [0, 0, -m_i g]^\top + \mathbf{f}_{\text{terrain}}(\mathbf{p}_i), & \text{if } \mathbf{p}_i \text{ is in contact,} \\
    [0, 0, -m_i g]^\top, & \text{otherwise.}
\end{cases}
\]
Here, the terrain reaction force is decomposed as $\mathbf{f}_{\text{terrain}}(\mathbf{p}_i) = \mathbf{f}_i^{(n)} + \mathbf{f}_i^{(\text{tan})}$.

The normal component $\mathbf{f}_i^{(n)}$ is given by:
\[
\mathbf{f}_i^{(n)} =
\begin{cases}
    k_i \big( h_i - p_{z,i} \big)\mathbf{n}_i - d_i \big(\dot{\mathbf{p}}_i^\top \mathbf{n}_i\big)\mathbf{n}_i, & p_{z,i} \leq h_i \\
    0, & \text{otherwise}
\end{cases}
\]

where $h_i$, $k_i$, and $d_i$ are the values of the predicted support height, stiffness, and damping maps-that is, $h_i = \phi_\mathrm{height}(\mathbf{p}_i)$, $k_i = \phi_\mathrm{stiffness}(\mathbf{p}_i)$, and $d_i = \phi_\mathrm{damping}(\mathbf{p}_i)$---evaluated at the position of mass point $\mathbf{p}_i$; $\mathbf{n}_i$ is the local surface normal.

And tangential (traction) forces for each point are
$\mathbf{f}_i^{(\mathrm{tan})} = \begin{bmatrix}
f_{\mathrm{long},i} & f_{\mathrm{lat},i}
\end{bmatrix}^\top$,
with $f_{\mathrm{long},i} = \mu_i m_i g\, [\sigma(u - v_x) - 0.5]$ and $f_{\mathrm{lat},i} = \mu_i m_i g\, [\sigma(-v_y) - 0.5]$, where $\mu_i = [\phi_\mathrm{friction}]_i$ is the local friction coefficient, $u$ is the commanded linear speed, and $v_x$, $v_y$ are the $x$ and $y$ components of the velocity of $\mathbf{p}_i$ in the local robot (body) frame.

Given the robot state, and the contact points $\mathbf{p}_i$, consequently the contact forces $\mathbf{f}_i$ are calculated, and ODE defined in (\ref{eq:physics}) is solved to find the next state, and computation repeats to calculate the whole trajectory.

\subsection{Conjugate Gradient (CG)}
\label{Sc. conjugate gradient}

The conjugate gradient (CG) algorithm~\cite{saad2003iterative} is an efficient, iterative method for solving large linear systems of the form $A \mathbf{x} = \mathbf{b}$, with $A \in \mathbb{R}^{n \times n}$ symmetric and positive semi-definite, and $\mathbf{b} \in \mathbb{R}^n$. Iterative CG algorithm relies only on repeatedly applying a matrix-vector multiplication function (``matvec'') implementing $\mathbf{x} \mapsto A \mathbf{x}$. Convergence is determined by the reduction of the residual norm relative to the initial value.

\subsection{Structured Convolution Matrix}
\label{Sc. toeplitz_matrix}

A \emph{Toeplitz matrix} is a matrix in which each descending diagonal from left to right is constant. This special structure arises when expressing the convolution operation in matrix form.

\begin{definition}[2D Convolution as Toeplitz Matrix Multiplication]
\label{def:toeplitz_matrix}
As described in~\cite{gray2006toeplitz}, let $X \in \mathbb{R}^{H \times W}$ be a 2D image and $\mathbf{x} = \mathrm{vec}(X) \in \mathbb{R}^n$ its row-major vectorization ($n = H \times W$). For a fixed convolution kernel $g \in \mathbb{R}^{k \times k}$, the standard 2D convolution $Y = g * X$ can be equivalently written as
\[
\mathbf{y} = L \mathbf{x},
\]
where $\mathbf{y} = \mathrm{vec}(Y) \in \mathbb{R}^n$, and $L \in \mathbb{R}^{n \times n}$ is a structured Toeplitz (more precisely, block Toeplitz with Toeplitz blocks) matrix determined entirely by the kernel $g$ (and the convolution’s padding scheme), independent of $\mathbf{x}$.
\end{definition}

\section{ProTerrain Methodology}
\label{sec:method}
To address the problem formulated in Section~\ref{sec:problem}, we introduce a novel structured formulation for covariance matrix that enables efficient, scalable modeling of spatially correlated aleatoric uncertainty via an implicit (matrix-free) approach. This formulation is coupled with a tractable negative log-likelihood (NLL) loss, allowing us to learn spatial correlations in high-dimensional settings (see Section~\ref{sc: correlated_uncertainty_estimation}).

We next describe our probabilistic terrain world modeling (described in Section \ref{Sc. Monoforce}) to produce spatially correlated uncertainty estimates for all terrain-relevant parameters using our devised structured probabilistic formulation (see Section~\ref{sec:probabilistic_world_architecture}). Subsequently, Section~\ref{Sc.traj_architecture} explains how the estimated uncertainty is rigorously propagated to trajectory distribution prediction by integrating samples through the differentiable physics engine.

Finally, we detail the end-to-end training pipeline, including the combination of our spatially correlated terrain losses with probabilistic trajectory supervision (see Section~\ref{Sc. losses}).

\subsection{Structured Aleatoric Uncertainty Formulation}
\label{sc: correlated_uncertainty_estimation}

We propose a core probabilistic formulation for spatially correlated aleatoric uncertainty by modeling each predicted terrain parameter map as a structured multivariate Gaussian distribution. This approach enables learning of realistic, spatially coherent uncertainty patterns across high-dimensional maps.

As described in section \ref{sec:problem}, learning probability distribution in (\ref{eq:multivariate_gaussian}) can be done via minimizing the standard multivariate negative log-likelihood (NLL) in (\ref{eq: general_multivariate_NLL}). However, directly inverting the full covariance matrix $\Sigma(\boldsymbol{z})$ or computing its determinant---as required by this loss function)---is computationally intractable in high dimensions. Instead, we reformulate the problem: rather than computing $\Sigma^{-1}$ explicitly, we solve the corresponding linear system $\Sigma \mathbf{a} = \mathbf{r}$ using the iterative conjugate gradient algorithm, which requires only matrix-vector products with $\Sigma$ (see Section~\ref{Sc. conjugate gradient}).

To further avoid the prohibitive cost of storing or learning a dense covariance matrix, we parameterize $\Sigma(\boldsymbol{z})$ implicitly to enforce spatial correlation through a convolutional structure. Specifically, we first predict a diagonal variance matrix $D$ (comprising independent per-cell variances), and then impose local spatial dependencies via convolution with a fixed kernel $g$. As shown in Definition~\ref{def:toeplitz_matrix}, such convolution can be represented as a matrix-vector multiplication by a structured Toeplitz matrix $L$ corresponding to $g$.

To ensure that $\Sigma$ is symmetric and positive semi-definite, we construct the spatially correlated covariance matrix as follows:

\begin{equation}
    \Sigma(\boldsymbol{z}) = L D L^\top,
    \label{eq:LDL}
\end{equation}
where $D \in \mathbb{R}^{n \times n}$ is a diagonal matrix containing the per-cell predicted variances ($D = \operatorname{diag}\left( {\sigma}^2(\mathbf{z}) \right)$), and $L \in \mathbb{R}^{n \times n}$ is the block Teoplitz matrix corresponding to $g$ as defined in Definition \ref{def:toeplitz_matrix}.

Consequently, given an arbitrary vector $\mathbf{x} = \mathrm{vec}(X)$ (the vectorized form of an input map $X$), any operation of the form $L \mathbf{x}$ (or $L^\top \mathbf{x}$) corresponds to a convolution of the input map with the kernel $g$ (or its flipped version $g^\top$):
\begin{equation}
\label{eq:mat_vec}
    L\, \mathrm{vec}(X) \equiv \operatorname{vec}(g * X).
\end{equation}

Similarly, with this structured covariance formulation, all required matrix-vector products involving $\Sigma$ can be performed implicitly, without explicitly constructing $\Sigma$ or $L$, as follows:
\begin{equation}
\label{eq:LDL_matvec}
\mathbf{x} \mapsto L D L^\top \mathbf{x} \; \Longleftrightarrow \; \mathbf{x} \mapsto \mathrm{vec}\left( g * (\, \sigma^2 \odot (g^\top * X) \,) \right),
\end{equation}
where $\sigma^2$ is the predicted variance map, $g^\top$ denotes a spatially flipped version of $g$, and $\odot$ is elementwise multiplication.

Next, we show how this structure yields to a computationally tractable, closed-form loss formulation for practical learning.

\begin{theorem}[NLL for Structured Covariance]
\label{thm:per_pixel_nll}
Let $\Sigma = L D L^\top$ be a structured covariance matrix, where $L \in \mathbb{R}^{n \times n}$ is a fixed matrix representing convolution with a kernel $g$, $D = \operatorname{diag}(\sigma^2)$ is a diagonal matrix of per-cell predicted variances, and $n = H' \times W'$ is the total number of spatial locations. 

NLL loss (\ref{eq: general_multivariate_NLL}) can be written as a sum of per-cell terms,
\begin{equation}
\label{eq:our_nll_loss}
\mathcal{L}_{\mathrm{NLL}} = \frac{1}{2} \sum_{i=1}^n \left( b_i^2 + \log \sigma_i^2 \right),    
\end{equation}

where $\mathbf{b} = D^{1/2} L^\top \mathbf{a}$ and $\mathbf{a} = \Sigma^{-1} \mathbf{r}$.
Evaluation of $\mathcal{L}_{\mathrm{NLL}}$ and its gradients only relies on the vectors $\mathbf{a}$ and $\mathbf{b}$,  which can be computed efficiently via convolutions and does not require explicitly forming or inverting $\Sigma$.
\end{theorem}

\begin{proof}

\textbf{Mahalanobis term.}

We can write the first term in NLL (\ref{eq: general_multivariate_NLL}) as:
\[
\mathbf{r}^\top \Sigma^{-1} \mathbf{r} = \mathbf{r}^\top \mathbf{a} = \mathbf{a}^\top \mathbf{r} = \mathbf{a}^\top (\Sigma \mathbf{a}) = \mathbf{a}^\top (L D L^\top \mathbf{a}).
\]
Define $\mathbf{v} = L^\top \mathbf{a}$, so $\mathbf{r}^\top \Sigma^{-1} \mathbf{r} = \mathbf{v}^\top D \mathbf{v}$. Let $\mathbf{b} = D^{1/2} \mathbf{v} = D^{1/2} L^\top \mathbf{a}$, then
\[
\mathbf{r}^\top \Sigma^{-1} \mathbf{r} = \mathbf{b}^\top \mathbf{b} = \sum_{i=1}^n b_i^2.
\]

\textbf{Log-determinant term.}
Since $D$ is diagonal and $L$ is fixed, the determinant factorizes~\cite{horn2012matrix}:
\[
\log \det \Sigma = \log \det(L D L^\top) = 2 \log \det L + \sum_{i=1}^n \log \sigma_i^2.
\]
As $L$ is fixed during training, the term $2 \log \det L$ can be treated as a constant and omitted from the loss. Combining the two terms and neglecting the constant terms, we proof (\ref{eq:our_nll_loss}).

To obtain $\mathbf{a}$, we solve the linear system
\[
\Sigma \mathbf{a} = \mathbf{r},
\]
where $\Sigma = L D L^\top$. Due to the high dimensionality, this is accomplished using an iterative CG method (decribed in section \ref{Sc. conjugate gradient}), which requires repeated matrix-vector multiplications with $\Sigma$, which can be computed without explicitly forming $\Sigma$ via convolution operators as described in (\ref{eq:LDL_matvec}).

Similarly, the vector $\mathbf{b}$ is computed as
\[
\mathbf{b} = D^{1/2} L^\top \mathbf{a} \quad\Longleftrightarrow\quad \mathbf{b} = \mathrm{vec}( \sigma \odot (g^\top * A) ),
\]
where $A$ is the 2D reshaped form of $\mathbf{a}$, and $\sigma$ is the predicted standard deviation map.

Thus, each loss computation involves $K$ conjugate gradient iterations (each with two convolutions and an elementwise multiplication), one additional convolution with $g^\top$ to compute $b$, and per-cell arithmetic for the closed-form loss. In practice, $K \approx 50$ CG iterations are generally sufficient for convergence, yielding a total computational complexity of $O(K H' W')$ per training sample. This efficiency makes our method suitable for practical uncertainty-aware terrain estimation at high spatial resolution.

\end{proof}

Throughout this work, we use the notation $\mathcal{L}_{\mathrm{NLL}}[\phi,\, \phi^\star]$ to denote the negative log-likelihood (NLL) loss devised in Theorem \ref{thm:per_pixel_nll} between a predicted input-conditioned probability distribution of parameter map $\phi$ and its ground-truth counterpart $\phi^\star$.

\begin{figure*}[ht]
  \centering
  \includegraphics[width=\textwidth,keepaspectratio]{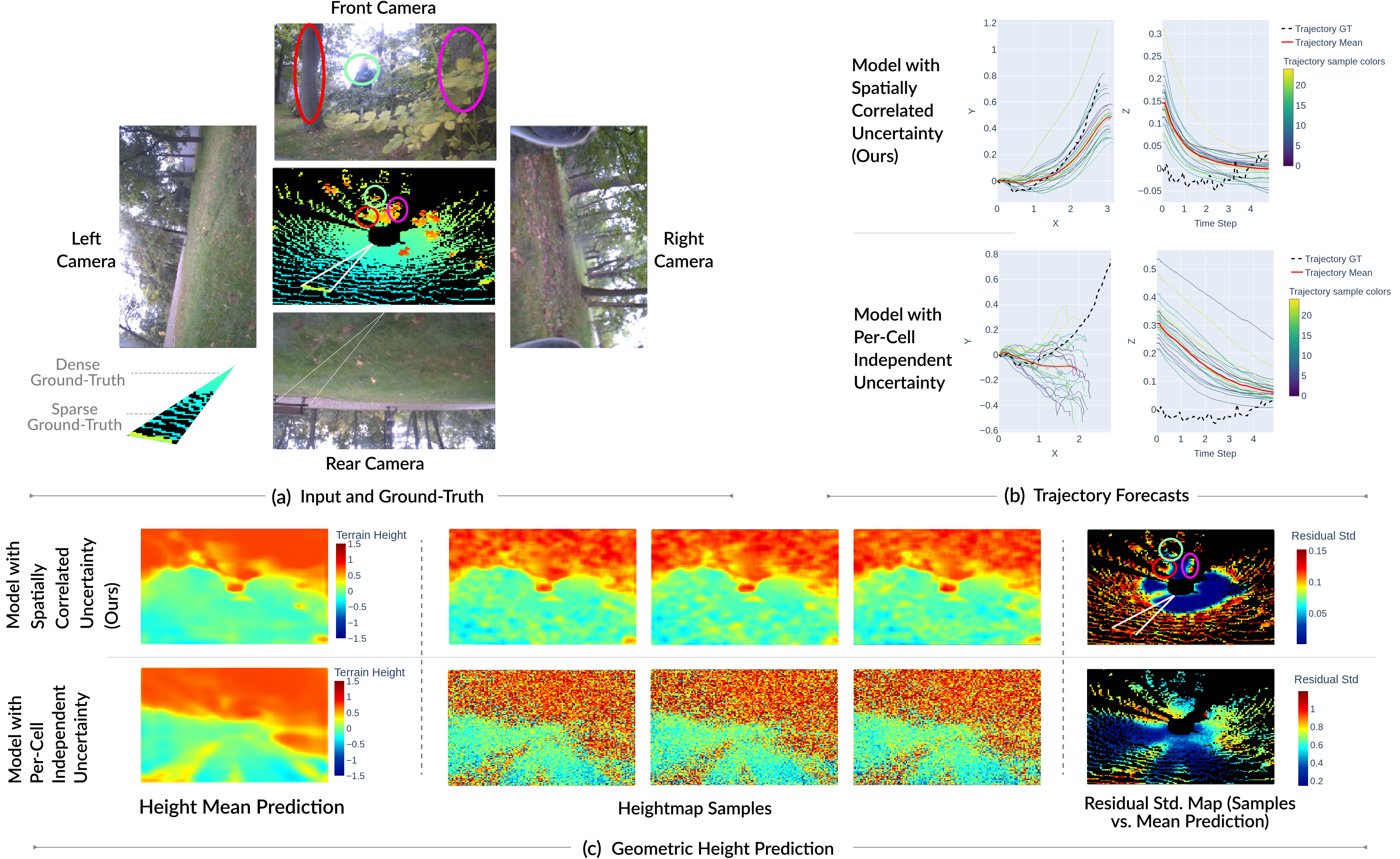}
  \caption{
\textbf{Qualitative comparison of our proposed method versus the standard per-cell independent uncertainty formulation.} (a) Ground-truth geometric heightmap (from LiDAR) within the camera field of view, with corresponding RGB input images. (b) Predicted mean geometric heightmap $\mu_{\phi_{\mathrm{geom}}}$, example probabilistic samples, and the residual standard deviation over 50 samples (visualizing overall uncertainty). (c) Predicted trajectory samples from both methods, compared with ground truth, visualized in the $xy$-plane and as height over time.
}
  \label{fig:comparison}
\end{figure*}

\subsection{Probabilistic World Model Architecture}
\label{sec:probabilistic_world_architecture}

We now describe our architecture for predicting a probabilistic, spatially correlated world model of terrain. As illustrated in Fig.~\ref{fig:detailed pipeline}, the network receives onboard images $\boldsymbol{z}$ and produces per-parameter spatial distributions, which are subsequently used as inputs to the physics engine that models terrain-robot interaction.

\paragraph{Backbone and Feature Extraction} 

We adopt the same backbone as MonoForce~\cite{agishev2024monoforce} for terrain encoding. Given onboard images $\boldsymbol{z}$, we employ the Lift-Splat-Shoot (LSS) framework~\cite{philion2020lift} for image-to-heightmap projection. LSS associates each image pixel with a camera ray, and along each ray the network predicts a discrete depth probability distribution (blue histogram in Fig.~\ref{fig:detailed pipeline}) together with per-depth visual embeddings. These features are accumulated along the ray to generate vertically-integrated feature maps, which are subsequently projected onto a 2D spatial grid representing the terrain. A convolutional encoder is then applied to extract visual features that capture local terrain properties.

\paragraph{Probabilistic Parameter Heads}
For every grid cell, the projected terrain features are further processed through deep convolutional heads to estimate, for each terrain parameter $\phi$, two output channels: a mean map $\mu_{\phi}$ and a (log-)variance map $\log \sigma^2_\phi$. We predict parameters $\phi \in \{\phi_{\mathrm{geom}},\; \phi_{\mathrm{height}},\; \phi_{\mathrm{stiffness}},\; \phi_{\mathrm{damping}},\; \phi_{\mathrm{friction}}\}$ as in MonoForce, sufficient for simulating robot-terrain interaction (see Section~\ref{Sc. Monoforce}).

\paragraph{Inducing Spatial Correlation}
To model spatial correlation between terrain cells, we impose local structure on the predicted per-cell variance maps by convolving the predicted variance map $\sigma^2_\phi(\boldsymbol{z})$ with a fixed 2D Gaussian kernel $g$. This yields a spatially correlated covariance matrix $\Sigma_{\phi} = L D_\phi L^\top$ for each parameter, as described in Section~\ref{sc: correlated_uncertainty_estimation}, and together with the mean $\mu_{\phi}$ forms our probabilistic world model $W_t$.

\subsection{Probabilistic Trajectory Prediction Architecture}
\label{Sc.traj_architecture}

Given the probabilistic world model $W_t$, initial robot state $s_0$, and control input sequence $\mathbf{u}_{1:T}$, we propagate the estimated uncertainty to predict trajectory distributions.

As illustrated in Fig.~\ref{fig:detailed pipeline}, for each spatially correlated terrain parameter distribution, we draw $M$ independent samples $\{\phi^{(m)}\}_{m=1}^M$ to capture uncertainty at the trajectory level. These samples are generated using the reparameterization trick with Monte Carlo sampling~\cite{kingma2013auto}:
\begin{equation}
    \mathrm{vec}(\phi^{(m)}) = \mathrm{vec}(\mu_{\phi}) + L \sqrt{D_\phi} \boldsymbol{\epsilon^{(m)}},
\quad \boldsymbol{\epsilon^{(m)}} \sim \mathcal{N}(0, 1),
\end{equation}
where $\boldsymbol{\epsilon^{(m)}}$ is a standard normal noise vector, scaled by the standard deviation matrix $\sqrt{D_\phi}$, and convolved with kernel $g$ (corresponding to $L$ in the matrix-vector formulation, see~(\ref{eq:mat_vec})).

Each sampled terrain parameter map $\phi^{(m)}$ (for $\phi \in \Phi$) is then used to construct world model sample $W_t^{(m)}$. Combined with initial state $s_0$ and control sequence $\mathbf{u}_{1:T}$, each $W_t^{(m)}$ is then processed by the differentiable physics engine (see~(\ref{eq:physics})), producing simulated trajectories $\tau^{(m)}$ and the associated forces.

At each time step $t = 1, \ldots, T$ and for each coordinate $d \in \{x, y, z\}$ of the robot position $\mathbf{x}_t$, we fit a univariate Gaussian distribution to the Monte Carlo trajectory samples. The sample mean and variance provide the point estimate $\mu_{t,d}$ and aleatoric uncertainty~$\sigma^2_{t,d}$ for each trajectory coordinate and time step, yielding our probabilistic trajectory forecast.

\subsection{Learning pipeline and Losses}
\label{Sc. losses}

We train the entire architecture by minimizing a weighted sum of losses, providing supervision on parameter maps $\phi \in \Phi$ in our terrain world model $W_t$ as well as on robot trajectory prediction:
\begin{equation}
\label{eq:loss_total}
\mathcal{L}_{\mathrm{total}} = \sum_{\phi \in \Phi}\lambda_{\phi} \mathcal{L}_{\phi} + \lambda_{\mathrm{traj}} \mathcal{L}_{\mathrm{traj}}.
\end{equation}

Here, $\{\lambda_\phi\}_{\phi \in \Phi}$ denotes a set of non-negative weighting coefficients controlling the relative contribution of each parameter map loss $\mathcal{L}_\phi$, and $\lambda_{\mathrm{traj}} \in \mathbb{R}_{+}$ is the weight for the trajectory loss term.

We supervise terrain parameter maps using our structured negative log-likelihood (NLL) formulation with spatial correlation (\ref{eq:our_nll_loss}), together with an explicit out-of-field-of-view (oFoV) variance regularization. The loss for each parameter map is given by:
\begin{equation}
    \mathcal{L}_{\phi} = \mathbf{W}_{\phi} \odot \mathcal{L}_{\mathrm{NLL}}[\phi_{\phi}, \phi_{\phi}^\star] 
    + \mathcal{L}_{\mathrm{oFoV}}[\log \sigma^2_{\phi}]
\end{equation}
where $\mathcal{L}_{\mathrm{NLL}}[\phi, \phi^\star]$ is the spatial NLL between predicted parameter $\phi$ and its ground-truth $\phi^\star$, computed pointwise at map coordinates that ground-truth data is available (i.e., within sensor field-of-view---indicated by $\mathrm{W_{\phi}}$ as a per-cell indicator). The oFoV regularization term $\mathcal{L}_{\mathrm{oFoV}}$ penalizes deviation of the predicted log-variance map from a calibrated prior $\sigma_0^2$ in regions lacking supervision (outside sensor FOV):

\begin{equation}
    \mathcal{L}_{\mathrm{oFoV}}[\log \sigma^2] = \lambda \; \mathbb{E}_{i \in \text{oFoV}} \left[ \left( \sigma^2_i - \sigma_0^2 \right)^2 \right],
\end{equation}
where the expectation is over all out-of-FOV locations, $\lambda$ is a regularization weight.

This design ensures robust learning: NLL is minimized only where ground-truth supervision is available, while log-variance in unsupervised (high-uncertainty) areas is encouraged to remain near the prior, thus avoiding degeneracy and yielding stable and interpretable uncertainty maps.

In the trajectory loss $\mathcal{L}_{\mathrm{traj}}$, the heteroscedastic Gaussian NLL from \cite{kendall2017uncertainties} is applied independently to each trajectory dimension $d \in \{x, y, z\}$ at each time step:
\begin{equation}
\label{eq:traj_NLL}
  \mathcal{L}_{\mathrm{traj}}
  = \sum_{t=1}^{T} \sum_{d \in \{x, y, z\}}
    \left[
      \frac{\bigl(\tau^{\star}_{t,d} - \mu_{t,d}\bigr)^2}{2\,\sigma_{t,d}^2}
      \;+\;\frac{1}{2}\log\sigma_{t,d}^2
    \right].
\end{equation}

This formulation enables the model to discover latent terrain properties that best explain observed robot motion, even in the absence of direct ground-truth supervision for those parameters.

\section{Experiments}
\label{Sc. eval}
We conduct extensive experiments in unstructured, real-world environments to evaluate our method’s ability to handle aleatoric uncertainties (e.g., sensor noise and occlusion) against baselines, with the goal of validating improvements in uncertainty modeling and trajectory forecasting.

\subsection{Dataset Overview}
\label{sec:dataset}

We evaluate our method on the publicly available dataset ROUGH\footnote{\url{https://github.com/ctu-vras/monoforce}}, which comprises several hours of off-road driving data recorded with two mid-sized tracked robotic platforms. The trajectories traverse a wide range of challenging terrains—including tall grass, overhanging branches, mud, and dense undergrowth. Each vehicle is outfitted with an Ouster OS0-128 LiDAR, as well as Basler RGB cameras. Robot poses are obtained via an ICP-based SLAM pipeline~\cite{pomerleau2013comparing}, which provides more reliable localization than GPS under forest canopy.

\subsection{Model Learning Setup}
For learning pipeline, we limit our probabilistic formulation and direct supervision to geometric, and support height terrain parameters ($\phi_\text{geom}$, $\phi_\text{height}$), leaving other terrain parameters in a deterministic format, getting self-supervised from trajectory loss. For $\phi_{geom}$, Lidar-based heightmap is used as ground-truth  and the ground truth of $\phi_{height}$ has been constructed using backpropagating robot traversed trajectories into map (similar to monoforce \cite{agishev2024monoforce}).  Height maps are represented as $128 \times 128$ grids at $0.1$ m per-cell resolution, offering sufficiently fine detail for local terrain modeling and demonstrating the scalability of our approach to high-dimensional representations.

\subsection{Qualitative Evaluation}
\label{sec:qualitative}

Figure~\ref{fig:comparison} shows qualitative results on a held-out test sequence, comparing our probabilistic approach with a standard per-cell independent model using the Gaussian negative log-likelihood formulation of~\cite{kendall2017uncertainties} (both trained with identical hyperparameters and network architecture).

As seen in panels (b) and (c), our model produces smoothly varying heightmap samples in regions of high uncertainty, while the per-cell baseline yields unrealistic, discontinuous variations. Consequently, our trajectory samples are coherent and closely follow the ground truth, whereas the per-cell model produces jittery and noisy forecasts.

Our model expresses lower uncertainty near the robot—where LiDAR supervision is strongest—and higher uncertainty in regions with sparse supervision (see white triangle in the ground-truth and residual Std. maps in Fig.~\ref{fig:comparison}). Elevated uncertainty is also observed in front camera regions with high vegetation and occlusions, accurately reflecting perceptual ambiguities (see colored ovals). In contrast, the per-cell model fails to consistently capture these uncertainty patterns in occluded or sparsely supervised areas.

\begin{table}[htbp]
\centering
\caption{Trajectory performance under different supervision settings. 
Abbreviations: Det.=Deterministic, Indep.=Independent, SC=Spatially Correlated (ours).}
\label{tab:traj}
\begin{tabular}{llcccc}
\toprule
Supervision & Method & ATE ${\downarrow}$ & bestATE ${\downarrow}$ & ECPE ${\downarrow}$ & ES ${\downarrow}$ \\
\midrule
Geom+Sup. 
& Det. & 0.6587 & - & - & - \\
& Indep. & 0.7198 & 0.7165 & 0.5593 & 4.8279 \\
& SC & \textbf{0.6386} & \textbf{0.6365} & \textbf{0.5406} & \textbf{4.3369} \\
\midrule
Geom
& Det. & \textbf{0.6550} & - & - & - \\
& Indep. & 2.4277 & 2.3845 & 0.5514 & 15.7464 \\
& SC & 0.6588 & \textbf{0.6585} & \textbf{0.5276} & \textbf{4.3079} \\
\midrule
Sup. Height 
& Det. & 0.6569 & - & - & - \\
& Indep. & 0.7167 & 0.7137 & 0.5762 & 4.8346 \\
& SC & \textbf{0.6492} & \textbf{0.6475} & \textbf{0.5242} & \textbf{4.4350} \\
\bottomrule
\end{tabular}
\end{table}


\subsection{Quantitative Evaluation}
\label{sec:quantitative}

We quantitatively benchmark our spatially correlated uncertainty formulation (SC) against two baselines—Deterministic (Det.) and Independent per-cell Gaussian (Indep.)—under three terrain supervision regimes: (i) both geometry and support height supervised (Geom+Sup.), (ii) geometry supervised only (Geom), and (iii) support height supervised only (Sup. Height). All models share the same architecture, physics engine, and training settings.

Trajectory forecasting performance is assessed using Mean Absolute Trajectory Error (ATE), BestATE (minimum over 10 Monte Carlo samples), Energy Score (ES)~\cite{gneiting2007strictly}, and \textcolor{black}{Expectation of Coverage Probability} Error (ECPE)~\cite{cui2020calibrated}. 

Table~\ref{tab:traj} summarizes the results. Across all supervision regimes,  our spatially correlated formulation (SC) yields significantly improved trajectory accuracy (ATE, BestATE), better uncertainty calibration (lower ECPE), and sharper probabilistic predictions (lower ES) relative to the Indep. approach, while matching the deterministic approach (Det.) in accuracy. Note that our uncertainty-aware framework enables principled risk-aware planning and robust navigation, which are unattainable with deterministic models.

A lower ES indicates that predicted trajectory distributions are closer to the ground truth and more concentrated around likely trajectories,  reducing overly diffuse or misleading uncertainty estimates and improving the representation of trajectory variability. Similarly, lower ECPE suggests the model’s confidence estimates are better aligned with observed outcomes, though some calibration gap may remain. Together, improvements in ES and ECPE provide more trustworthy uncertainty estimates, supporting better downstream risk-estimate and safer planning in complex environments.




\section{Conclusion}
\label{Sc.conclusion}

We introduced a tractable probabilistic framework for rough terrain modeling that captures spatially correlated aleatoric uncertainty in physics-informed parameters. Our structured covariance approach allows efficient, end-to-end uncertainty propagation from perception to robot trajectory prediction. Our experiments show that modeling spatial correlation significantly improves trajectory accuracy and calibration over uncertainty estimation baselines, underscoring the value of realistic uncertainty modeling in complex, real-world environments. Future work will address epistemic uncertainty, closed-loop control, and richer terrain semantics.

\bibliographystyle{IEEEtran}
\bibliography{refs}

\end{document}